%% file: main.tex
\PassOptionsToClass{12pt}{jmlr}
\documentclass{colt2026}

\jmlrproceedings{Preprint}{Preprint}
\jmlrvolume{}
\jmlrpages{}
\jmlryear{2026}

\usepackage{times}
\usepackage{mathrsfs}
\usepackage{microtype}
\usepackage{xcolor}
\usepackage{booktabs}
\usepackage{array}
\usepackage{enumitem}
\usepackage{needspace}
\usepackage[nameinlink,capitalize,noabbrev]{cleveref}
\crefname{lemma}{Lemma}{Lemmas}
\Crefname{lemma}{Lemma}{Lemmas}
\crefname{corollary}{Corollary}{Corollaries}
\Crefname{corollary}{Corollary}{Corollaries}
\AddToHook{env/theorem/begin}{\crefalias{section}{theorem}}
\AddToHook{env/lemma/begin}{\crefalias{section}{lemma}}
\AddToHook{env/corollary/begin}{\crefalias{section}{corollary}}
\setlist[enumerate]{itemsep=1pt,topsep=2pt,parsep=0pt,partopsep=0pt}
\newcommand{\R}{\mathbb R}
\newcommand{\E}{\mathbb E}
\newcommand{\Pp}{\mathbb P}
\newcommand{\Sph}{\mathbb S}
\newcommand{\ind}{\mathbf 1}
\newcommand{\cD}{\mathcal D}
\newcommand{\cR}{\mathcal R}
\newcommand{\cT}{\mathcal T}
\newcommand{\cU}{\mathcal U}
\newcommand{\VC}{\operatorname{VC}}
\newcommand{\cone}{\operatorname{cone}}
\newcommand{\conv}{\operatorname{conv}}
\newcommand{\aff}{\operatorname{aff}}
\newcommand{\spanop}{\operatorname{span}}
\newcommand{\tr}{\operatorname{tr}}
\newcommand{\Lam}{\Lambda}
\newcommand{\NRE}{\mathsf N_{\mathrm{RE}}}
\newcommand{\SB}{\mathsf Q}
\newcommand{\Normal}{\mathsf N}
\newcommand{\HT}{\mathsf H}
\newcommand{\Gauss}{\mathsf G}
\newcommand{\eps}{\varepsilon}
\newcommand{\defeq}{\mathrel{:=}}

\definecolor{AcademicBlue}{RGB}{0,70,125}
\hypersetup{
  colorlinks=true,
  linkcolor=AcademicBlue,
  citecolor=AcademicBlue,
  urlcolor=AcademicBlue,
  pdftitle={Restricted Eigenvalues Beyond Gaussian Width: Threshold Occupancy under Heavy Tails},
  pdfauthor={Shi Fu, Huibo Xu, Qixin Zhang, Dacheng Tao},
  pdfkeywords={restricted eigenvalues, heavy-tailed designs, small-ball method, Gaussian width, VC dimension}
}

\makeatletter
\renewcommand*{\@titlefoot}{\scriptsize Preprint, 2026.\hfill}
\makeatother

\title[Heavy-Tailed Restricted Eigenvalues]{Restricted Eigenvalues Beyond Gaussian Width:\\
Threshold Occupancy under Heavy Tails}
\coltauthor{\Name{Shi Fu} \and
  \Name{Huibo Xu} \and
  \Name{Qixin Zhang} \and
  \Name{Dacheng Tao}\\
  \addr Generative AI Lab, College of Computing and Data Science,\\
  Nanyang Technological University, Singapore}

\begin{document}
\maketitle

\begin{abstract}
Restricted eigenvalue (RE) bounds govern stable recovery by norm-regularized estimators.  For isotropic sub-Gaussian measurements, the benchmark sample size is $1+w(A)^2$, where $w(A)$ is the Gaussian width of the normalized descent cone.  The COLT 2015 open-problem note \citep{banerjee2015open} asked whether the same law follows for heavy-tailed designs from a uniform small-ball condition alone.  We give an explicit and systematic negative answer to the general question as formulated there: the proposed law fails in its full dimension-free, arbitrary-set form, and the missing obstruction is simultaneous threshold occupancy.  A constant-width polyhedral descent cone with fixed small-ball constants has zero empirical RE on every sample path up to half the ambient dimension.  More generally, every finite range space admits exact threshold encoding in an arbitrarily narrow spherical cap and a lift to a full polyhedral descent-cone section.  For every fixed threshold VC dimension $d$, as $\beta\downarrow0$, the sharp worst-case sample complexity is $\Theta(\beta^{-1}[d\log(1/\beta)+\log(1/\delta)])$.  The separation persists under exact isotropy and all finite moments: on the same constant-width cone, Gaussian measurements succeed with $O(1+\log(1/\delta))$ samples, whereas an isotropic heavy-tailed design fails pathwise for $n\lesssim\sqrt{p/\log p}$.  Gaussian smoothing yields an everywhere-positive $C^\infty$ density while retaining arbitrarily poor RE.  Under isotropy, a distribution-free fallback governed by affine dimension times squared enclosing radius is sharp on this family.
\end{abstract}


\section{Introduction}

Many high-dimensional estimators recover a structured parameter by minimizing a norm subject to fitting the data.  The Lasso is the canonical example: the $\ell_1$ norm favors sparse vectors, but it cannot distinguish the target $\theta^\star$ from a perturbation that both preserves the measurements and does not increase the norm.  More generally, the perturbations that a regularizer cannot exclude form its \emph{descent cone}.  Recovery is possible only when the measurement matrix is well conditioned on this cone.

The restricted eigenvalue (RE) quantifies this condition.  Let $X\in\R^{n\times p}$ have i.i.d. rows $Z_1,\ldots,Z_n$ with common law $P$, and let $A\subseteq\Sph^{p-1}$ be the normalized set of relevant directions, typically the spherical section of a descent cone.  We write
\begin{equation}\label{eq:lambda-def}
  \Lam_n(X;A)=\inf_{u\in A}\frac{\|Xu\|_2^2}{n}
  =\inf_{u\in A}\frac1n\sum_{i=1}^n\langle Z_i,u\rangle^2.
\end{equation}
If $\Lam_n(X;A)>0$, no relevant direction lies in the kernel.  A quantitative lower bound controls the sensitivity of recovery to noise and, in statistical models, the estimation error \citep{bickel2009simultaneous,chandrasekaran2012convex,negahban2012unified}.

For centered isotropic sub-Gaussian rows, a clean geometric law holds up to distribution-dependent constants: the benchmark sample size is controlled by $1+w(A)^2$, where
$w(A)=\E\sup_{u\in A}\langle g,u\rangle$ is the Gaussian width.  The $w(A)^2$ term acts as an effective dimension of the cone, while the additive constant covers even zero-width sets.  For Gaussian measurements, statistical-dimension theory sharpens the picture to precise phase transitions for sparse recovery, low-rank recovery, and many other convex inverse problems \citep{gordon1988escape,chandrasekaran2012convex,amelunxen2014living,tropp2015convex}.

Heavy-tailed measurements motivate a different route.  Upper-tail concentration may be unavailable, but each fixed direction may still be visible with nontrivial probability.  The small-ball condition asks that, for some $\alpha,\beta>0$,
\begin{equation}\label{eq:smallball}
  \SB_\alpha(P,A)\defeq
  \inf_{u\in A}P\{ |\langle Z,u\rangle|\ge\alpha\}\ge\beta.
\end{equation}
The apparent analogy hides a change in quantifiers.  The small-ball condition says that, for every fixed direction $u$, a fresh population draw detects $u$ with probability at least $\beta$.  An RE bound asks for more: with high probability, one common sample must detect every direction in $A$ simultaneously.  Under sub-Gaussian assumptions, increment control couples nearby directions and Gaussian width measures the remaining uniformity cost.  A bare small-ball condition provides no such coupling.

The COLT 2015 open problem of \citet{banerjee2015open} asked whether, for arbitrary spherical $A$, one can nevertheless prove
\begin{equation}\label{eq:putative-bound}
  \inf_{u\in A}\|Xu\|_2^2
  \ge c_1(\alpha,\beta)n-c_2(\alpha,\beta)w(A)^2
\end{equation}
with high probability.  We make explicit the sample-complexity interpretation of the ``suitable constants'' in that question: $c_1>0$ and $c_2<\infty$ may depend on the small-ball parameters but not on the ambient dimension or on $A$.  Without this dimension-free uniformity, the proposed Gaussian-width law would not express the same-order sample-complexity principle posed in the original note.  A positive answer would have made the Gaussian-width phase diagram universal far beyond Gaussian designs, without upper-tail concentration or high-order moment assumptions.

\paragraph{Our answer.}
We show that the dimension-free Gaussian-width law posed in the COLT 2015 note does not follow from a bare small-ball condition, even when the design is exactly isotropic.  Our results identify simultaneous threshold occupancy as the general obstruction, determine its sharp worst-case complexity, and delineate what isotropy does and does not recover.

\paragraph{The missing information is simultaneous coverage.}
Gaussian width measures Euclidean proximity between directions, but a small-ball hypothesis need not couple their visibility events.  Imagine many latent row types and one direction for each subset of these types.  Direction $u$ is visible precisely when the sampled row type belongs to its associated subset.  All directions can be placed in an arbitrarily narrow spherical cap, so their Gaussian width is tiny, while a finite sample must still hit every subset in a complicated range system.  Missing one range leaves a direction with zero empirical energy.  Thus the problem is not Gaussian-process geometry but threshold coverage.

\paragraph{Main contributions.}
We establish four complementary results that expose the obstruction, quantify its complexity, and identify a positive boundary.

\begin{itemize}[leftmargin=1.5em]

\item \emph{A pathwise counterexample at constant width.}
For every even $m$, we give a polyhedral norm and a centered, directionally heavy-tailed design with $w(A)=O(1)$ and fixed small-ball constants, yet $\Lam_n(X;A)=0$ on every sample path whenever $n\le m/2$ (\Cref{thm:explicit}).
This is a deterministic finite-sample obstruction, not a failure event hidden in a tail bound.

\item \emph{Range-space universality and the sharp low-mass law.}
We represent every finite probability range space exactly by threshold events at arbitrarily small Gaussian width (\Cref{thm:exact-range}), then lift missed-range witnesses to a full polyhedral descent-cone section while preserving uniform small-ball mass over the cone (\Cref{thm:cone-range}).  Thus occupancy is intrinsic.  At threshold VC dimension $d$, uniform RE holds at
\(n\gtrsim\beta^{-1}\{d\log(2/\beta)+\log(1/\delta)\}\),
and this order is necessary for every fixed $d$ as $\beta\downarrow0$ (\Cref{thm:sharp-vc}).  This sharp law concerns general spherical sets; the cone lift does not claim to preserve the generating range space's VC dimension.

\item \emph{An isotropic and smooth separation on the same geometry.}
On one common descent cone, Gaussian measurements succeed after $O(1+\log(1/\delta))$ samples, whereas a centered isotropic heavy-tailed design with all moments has zero RE pathwise for $n\lesssim\sqrt{p/\log p}$ (\Cref{thm:isotropic,cor:constant-isotropic}).  Gaussian smoothing gives a positive $C^\infty$ density while leaving the restricted singular value arbitrarily small (\Cref{cor:smooth}).
Thus neither covariance normalization nor removing atoms restores Gaussian-width control.

\item \emph{A positive boundary under isotropy.}
Isotropy yields a coarser distribution-free guarantee: if $q(A)$ is the affine dimension and $r(A)$ the minimum enclosing radius, then $q(A)r(A)^2$ replaces $w(A)^2$ in a small-ball RE bound (\Cref{thm:radius,cor:radius-sample}).  The anchored-frame family matches this scale.
This is a genuine positive consequence of isotropy, without asserting a joint minimax formula for all complexity pairs.

\end{itemize}

\paragraph{Relation to earlier sparse-recovery constructions.}
The spiky construction in the 2014 arXiv version of \citet{lecue2017sparse} predates the COLT note and can be retrospectively interpreted as a negative instance of the fully uniform Gaussian-width implication for a particular $\ell_1$-related sparse-recovery geometry. The COLT note cited that work in its discussion of unit $s$-sparse directions and, after reviewing it together with the threshold-VC approach of \citet{koltchinskii2015bounding}, nevertheless stated that the known results covered only ``certain special cases of $A$'' and that the problem for general $A$ under the small-ball property remained open \citep[p.~1754]{banerjee2015open}.  Our contribution is a structural treatment of that arbitrary-set formulation: constant-width pathwise failure, exact encoding of arbitrary finite range spaces, the sharp occupancy law, and same-geometry isotropic and smooth separations.  \Cref{sec:related} gives the full comparison.

\paragraph{Scope.}
Our impossibility concerns what follows uniformly from marginal small-ball information alone.  Moment or increment assumptions and sparsity-specific structure can couple nearby directions and thereby exclude our unrestricted encodings.  The appendices provide all constants, measurability conventions, and proofs.

\section{Problem setup and proof ideas}\label{sec:overview}

Let $V$ be a finite-dimensional Euclidean space, write $S(V)=\{v\in V:\|v\|_2=1\}$, and let $A\subseteq S(V)$ be nonempty.  For a norm $\cR$ and $\theta\ne0$, set
$\cD(\cR,\theta)=\{h:\cR(\theta+th)\le\cR(\theta)\text{ for some }t>0\}$.  All cones below are closed and polyhedral.

We use the fixed radial variable $L=1+e^{G_0}$, where $G_0\sim\Normal(0,1)$.  It has finite moments of every order but $\E e^{sL}=\infty$ for every $s>0$.  An independent Rademacher sign centers our designs.  We call a random vector \emph{directionally heavy-tailed} if
\begin{equation}\label{eq:directional-heavy-tail}
  \E e^{s|\langle Z,v\rangle|}=\infty
  \qquad\text{for every }v\ne0\text{ and every }s>0.
\end{equation}
For range-space encodings with possibly degenerate marginals, ``heavy-tailed radial multiplier'' refers to this same $L$.

The proof first encodes incidence in a narrow cap and then lifts it to a descent cone; isotropic negative and positive extensions complete the picture.

\paragraph{Narrow-cap incidence encoding.}
For a finite range space $\{C_1,\ldots,C_M\}$, place each $u_j$ near one common anchor and let a row store $(\ind\{Y\in C_j\})_{j\le M}$.  Then $|\langle Z,u_j\rangle|\ge1$ exactly when $Y\in C_j$.  Shrinking the perturbation drives width to zero without changing these events.

\paragraph{From finitely many directions to a descent cone.}
A finite spherical set is not yet a regularizer geometry.  Transporting the positive orthant through a narrow embedding yields a polyhedral norm whose cone directions are normalized mixtures of the generators.  Averaging preserves uniform small-ball mass over the cone, while a missed range annihilates an extreme ray.

\paragraph{Restoring isotropy without losing the witness.}
Start with polynomially many anchored Rademacher row types having near-identity covariance, constant slab mass, and well-conditioned small submatrices; whitening gives a tight frame.  After observing types $D$, normalize the residual obtained by projecting the anchor off their span.  It annihilates every observed row and stays within $O(\sqrt{|D|/p})$ of the anchor.  A cap-to-cone lemma controls the full cone's width.

\paragraph{Why isotropy still yields a positive theorem.}
For a centered isotropic design, translate $A$ by a minimum enclosing center $a$ and project the symmetrized process onto the linear space $\spanop(A-a)$ parallel to $\aff(A)$.  Its expected supremum is at most $r(A)\sqrt{q(A)}$, giving the dimension--radius bound.  On our family, $q(A)=p$ and $r(A)^2\asymp k/p$, matching the pathwise scale $k$.

\section{A constant-width pathwise counterexample}\label{sec:counterexample}

The first theorem already rules out \eqref{eq:putative-bound}.  Its role is to expose the coverage obstruction in an elementary form before the universal encoding.

\Needspace{8\baselineskip}
\begin{theorem}[Explicit constant-width counterexample]\label[theorem]{thm:explicit}
For every even $m\ge2$, there are a Euclidean space $V_m$ of dimension $m$, a polyhedral norm $\mathcal R_m$, a point $\theta_m^\star$, and a centered, directionally heavy-tailed $Z_m\in V_m$ with law $P_m$.  Set $A_m=\cD(\mathcal R_m,\theta_m^\star)\cap S(V_m)$.  Then
\begin{equation}\label{eq:explicit-properties}
  \SB_{1/(2\sqrt2)}(P_m,A_m)\ge\frac13,
  \qquad w(A_m)\le2\sqrt{2/\pi},
  \qquad \VC(\cT_1(A_m)|_{\operatorname{supp}P_m})\ge\frac m2.
\end{equation}
Moreover, for every $n\le m/2$ and every realization of the sample,
\begin{equation}\label{eq:explicit-zero}
  \Lam_n(X;A_m)=0.
\end{equation}
The radial law $L$ is the same in every dimension.
\end{theorem}

Here $\cT_\alpha(A)=\{\{z:|\langle z,u\rangle|\ge\alpha\}:u\in A\}$.  The construction separates an anchor coordinate $t$ from a balanced perturbation $y$.  Work in $V_m=\{(t,y/m):\sum_jy_j=0\}$ and define
\begin{equation}\label{eq:explicit-construction}
\begin{aligned}
  C_m&=\{(t,y/m):t\ge0,\ |y_j|\le t\},\\
  \mathcal R_m(t,y/m)&=\max\{|t|,\max_j|t-y_j|,\max_j|t+y_j|\}.
\end{aligned}
\end{equation}
The norm makes $C_m=\cD(\mathcal R_m,(-1,0))$.  Since $\|y/m\|_2\le t/\sqrt m$, its normalized directions occupy an $O(m^{-1/2})$ cap, explaining the constant width.  For visibility, let $J$ be uniform on $[m]$ and $\sigma$ an independent Rademacher sign, set $r_j=e_0+me_j-\sum_{\ell=1}^me_\ell$, and take $Z_m=\sigma Lr_J$.  At $v=(t,y/m)\in C_m$, type $j$ has score $t+y_j$.  These scores lie in $[0,2t]$ and average to $t$, so at least $m/3$ are at least $t/2$.  Normalization gives the uniform small-ball bound.

For the pathwise witness, let $D$ be the distinct observed types.  When $|D|\le m/2$, the assignments $s_j=-1$ on $D$ extend to a balanced $s\in\{-1,1\}^m$.  Then $u_s\propto(1,s/m)$ belongs to $A_m$, and every observed score vanishes because $\langle r_j,u_s\rangle\propto1+s_j=0$.  Thus one data-dependent direction annihilates the sample.  Extending arbitrary signs on any prescribed $m/2$ types also gives shattering.  \Cref{app:explicit} supplies the norm, width, tail, and boundary calculations.

\begin{corollary}[Uniform failure of the Gaussian-width law]\label[corollary]{cor:negative}
Fix any $c_1>0$, $0\le c_2<\infty$, and integer $n_0\ge1$.  There are $n\ge n_0$, a spherical descent-cone section $A$, and a design with the fixed small-ball constants in \Cref{thm:explicit} for which
$\inf_{u\in A}\|Xu\|_2^2<c_1n-c_2w(A)^2$ with probability one.
\end{corollary}

Indeed, take $m=2n$ and then $n$ large enough.  Thus a cone that has constant complexity according to Gaussian width can require a number of heavy-tailed measurements proportional to its ambient dimension.  The obstruction is not a rare failure of concentration.

\section{Threshold occupancy governs the small-ball worst case}\label{sec:occupancy}

For $u\in A$, let $I_\alpha(u)=\{i\in[n]:|\langle Z_i,u\rangle|\ge\alpha\}$ collect the rows that $\alpha$-detect $u$, and write $\widehat q_{n,\alpha}(A)=n^{-1}\inf_{u\in A}|I_\alpha(u)|$.  Every detected row contributes at least $\alpha^2$ to the empirical energy, so deterministically
\begin{equation}\label{eq:occ-suff}
  \Lam_n(X;A)\ge\alpha^2\widehat q_{n,\alpha}(A).
\end{equation}
Thus uniform threshold coverage is sufficient.  It is not pointwise necessary for a fixed law---a few unusually large observations may still supply quadratic energy---but the results below give the relevant minimax converse: every finite coverage obstruction can be realized as an RE obstruction at arbitrarily small Gaussian width.

\subsection{Arbitrary range spaces in narrow caps and descent cones}

Let $(\Omega,\mu,\mathscr C)$ be a finite probability range space with full support, meaning $\Omega=\operatorname{supp}(\mu)$, and let $\mathscr C=\{C_1,\ldots,C_M\}$ with $\mu(C_j)\ge\beta$.  Its VC dimension is the largest number of row types on which all hit--miss labelings occur.  For a latent sample $Y_1,\ldots,Y_n$, write $\mu_n(C)=n^{-1}\sum_{i=1}^n\ind\{Y_i\in C\}$.

\Needspace{8\baselineskip}
\begin{theorem}[Exact range-space representation]\label[theorem]{thm:exact-range}
For every $\eta>0$ there are a centered design with the fixed heavy-tailed radial multiplier and a finite $A=\{u_1,\ldots,u_M\}\subset\Sph^M$ such that $w(A)\le\eta$, $\SB_1(P,A)\ge\beta$, and, on the design support,
\begin{equation}\label{eq:range-exact}
  \ind\{|\langle Z,u_j\rangle|\ge1\}=\ind\{Y\in C_j\}
  \qquad(j\in[M]).
\end{equation}
Consequently the support-restricted threshold class and $\mathscr C$ have the same VC dimension.  For every sample,
\begin{equation}\label{eq:clipped-exact}
  \inf_{u\in A}\frac1n\sum_{i=1}^n\min\{1,\langle Z_i,u\rangle^2\}
  =\min_{j\le M}\mu_n(C_j),
\end{equation}
and $\Lam_n(X;A)>0$ exactly when the latent sample hits every $C_j$.
\end{theorem}

The representation is explicit.  Take $u_j=(e_0+\eps e_j)/\sqrt{1+\eps^2}$ and
\[
  Z=\frac{\sqrt{1+\eps^2}}{\eps}\,\sigma L
    \sum_j\ind\{Y\in C_j\}e_j\in\R^{M+1}.
\]
Then $\langle Z,u_j\rangle=\sigma L\ind\{Y\in C_j\}$, while
$w(A)\le\eps\sqrt{2\log(2M)}$.  Decreasing $\eps$ hides an arbitrary incidence pattern in an arbitrarily narrow cap without changing a single threshold event.  This exact identity, rather than a comparison inequality, is what makes the subsequent lower bounds lossless.

The finite representation does not by itself establish a statement about norm-regularized recovery.  The next theorem shows that the same missed-range witness survives after passing to the full spherical section of a descent cone.

\Needspace{8\baselineskip}
\begin{theorem}[Descent-cone universality]\label[theorem]{thm:cone-range}
For every finite range space above and $\eta>0$, there are an explicit polyhedral norm $\mathcal R_\eps$ on an $M$-dimensional space $V_\eps$, a point $\theta_\eps^\star$, and a centered design with the fixed heavy-tailed radial multiplier such that, for $A_\eps=\cD(\mathcal R_\eps,\theta_\eps^\star)\cap S(V_\eps)$,
\begin{equation}\label{eq:cone-range-properties}
  w(A_\eps)\le\eta,
  \qquad \SB_1(P,A_\eps)\ge\frac{\beta}{2-\beta}.
\end{equation}
If a latent sample misses one $C_j$, then $\Lam_n(X;A_\eps)=0$.
\end{theorem}

For intuition, the injective map $T_\eps a=(\ind^\top a,\eps a)$ transports the $\ell_\infty$ norm to $V_\eps=T_\eps\R^M$.  Write $\Delta_M=\{q\in\R_+^M:\ind^\top q=1\}$ for the probability simplex.  At $\theta_\eps^\star=-T_\eps\ind$,
\begin{equation}\label{eq:cone-parametrization}
  \cD(\mathcal R_\eps,\theta_\eps^\star)=T_\eps\R_+^M,
  \qquad
  u(q)=\frac{e_0+\eps q}{\sqrt{1+\eps^2\|q\|_2^2}},\quad q\in\Delta_M.
\end{equation}
If $f_q(Y)=\sum_jq_j\ind\{Y\in C_j\}$, then $0\le f_q\le1$ and $\E f_q\ge\beta$, so $P\{f_q\ge\beta/2\}\ge\beta/(2-\beta)$.  The scaling of $Z$ turns this event into the threshold at one.  A missed $C_j$ annihilates the generator $u(e_j)$, while a direct Gaussian calculation controls the full cone width.  See \Cref{app:cone-range}.

The obstruction is therefore not tied to one specially chosen cone: arbitrarily narrow Euclidean geometry can hide any finite combinatorial visibility pattern.

\subsection{The VC envelope and its sharp low-mass regime}

For a given law $P$, call $A$ \emph{pointwise measurable} when $\{\ind_C|_{\operatorname{supp}P}:C\in\cT_1(A)\}$ has a countable subclass whose pointwise sequential closure contains the whole class.  Let $\NRE(d,\beta,\delta)$ be the least integer $n$ such that, for every $N\ge n$, every pointwise-measurable spherical set $A$ in a finite-dimensional Euclidean space, and every design satisfying $\SB_1(P,A)\ge\beta$ and
$\VC(\cT_1(A)|_{\operatorname{supp}P})\le d$, one has $\Lam_N(X;A)\ge\beta/2$ with probability at least $1-\delta$.  The requirement is marginal for each $N$.  No event simultaneous over all $N$ is asserted.

\Needspace{8\baselineskip}
\begin{theorem}[VC envelope and sharp fixed-$d$ low-mass law]\label[theorem]{thm:sharp-vc}
There are universal $c,C>0$ such that, for every integer $d\ge1$ and $0<\beta,\delta\le1/4$,
\begin{equation}\label{eq:sharp-vc-general}
  c\,\frac{d+\log(2/\beta)+\log(1/\delta)}{\beta}
  \le \NRE(d,\beta,\delta)
  \le C\,\frac{d\log(2/\beta)+\log(1/\delta)}{\beta}.
\end{equation}
Moreover, for every fixed $d\ge1$ there is $\beta_d>0$ such that, whenever $0<\beta\le\beta_d$,
\begin{equation}\label{eq:sharp-vc}
  \NRE(d,\beta,\delta)
  \ge c\,\frac{d\log(2/\beta)+\log(1/\delta)}{\beta}.
\end{equation}
Thus the upper bound is optimal up to universal multiplicative constants in the low-mass regime for each fixed VC dimension.
\end{theorem}

The upper bound is a constant-relative-error VC approximation followed by \eqref{eq:occ-suff}.  The general lower envelope combines three transparent mechanisms: all $d$-subsets of a uniform ground set give $d/\beta$, disjoint singleton ranges give $\beta^{-1}\log(1/\beta)$ already at VC dimension one, and one range of mass $\beta$ gives $\beta^{-1}\log(1/\delta)$.  The fixed-$d$ product term follows from the sharp $\eps$-net lower bound of \citet{komlos1992almost}.  \citet{pach2013tight} exhibit the logarithm even for geometric VC-two systems.  \Cref{app:sharp-vc} states the needed consequence of the external result and proves every reduction.  The classical range-space inequalities are not new.  Their exact realization as arbitrarily narrow-width RE problems is the new structural statement.

\section{Isotropy does not restore Gaussian width}\label{sec:isotropic}

The preceding encodings use unconstrained row magnitudes.  One might therefore hope that exact covariance normalization reconnects Euclidean proximity with common visibility.  The following construction shows that this hope is false.

\Needspace{8\baselineskip}
\begin{lemma}[Anchored tight frame]\label[lemma]{lem:frame}
There are absolute $a_0,b_0,c_0,c_1,C_1>0$ such that, for all sufficiently large $p$ and $M=p^4$, vectors $b_1,\ldots,b_M\in\R^p$ and a unit $u_\star$ exist with
\begin{equation}\label{eq:frame-properties}
  \frac1M\sum_j b_j b_j^\top=I_p,
  \qquad \langle b_j,u_\star\rangle=1,
  \qquad \frac1M|\{j:|\langle b_j,v\rangle|\ge a_0\}|\ge b_0
\end{equation}
for every unit $v$.  If $B_D$ has rows $b_j^\top$, then, whenever $|D|\le c_0p/\log p$,
\begin{equation}\label{eq:gram}
  c_1pI_{|D|}\preceq B_D B_D^\top\preceq C_1pI_{|D|}.
\end{equation}
\end{lemma}

Start from $a_j=(1,\xi_j)$ with independent Rademacher vectors $\xi_j\in\{-1,1\}^{p-1}$.  Their empirical covariance is close to identity.  A uniform slab law gives global small-ball mass, while a net argument and a union bound make every small row submatrix well conditioned.  Whitening then yields the exact tight frame without losing the common anchor.  \Cref{app:frame} records the probability budget and the net-to-operator step.

For $D\subseteq[M]$ with $|D|\le c_0p/\log p$, let $P_D$ project onto $\spanop\{b_j:j\in D\}$, with $P_\varnothing=0$, and set
$u_D=(I-P_D)u_\star/\|(I-P_D)u_\star\|_2$.  The Gram bound gives
$\|u_D-u_\star\|_2^2\asymp|D|/p$ and $B_Du_D=0$.  Define
$C_{p,k}=\cone\{u_D:|D|\le k\}$ and $A_{p,k}=C_{p,k}\cap\Sph^{p-1}$.

\Needspace{8\baselineskip}
\begin{theorem}[Isotropic heavy-tail separation]\label[theorem]{thm:isotropic}
There are absolute $c,C,\alpha_0,\beta_0>0$ such that, for large $p$ and $1\le k\le cp/\log p$, the following statements hold.
\begin{enumerate}[label=(\roman*),leftmargin=1.7em]
\item $A_{p,k}$ is the spherical section of a polyhedral norm descent cone and
\begin{equation}\label{eq:isotropic-width}
  w(A_{p,k})\le C\left(k\sqrt{\frac{\log p}{p}}+\frac{k}{\sqrt p}\right).
\end{equation}
\item There are centered isotropic laws $P_{\HT}$ and $P_{\Gauss}=\Normal(0,I_p)$ with common absolute small-ball constants
\begin{equation}\label{eq:common-sb}
  \inf_{v\in\Sph^{p-1}}P_s\{|\langle Z,v\rangle|\ge\alpha_0\}\ge\beta_0,
  \qquad s\in\{\HT,\Gauss\}.
\end{equation}
The law $P_{\HT}$ has finite moments of every order and no positive exponential moment in any nonzero marginal, yet $\Lam_n(X_{\HT};A_{p,k})=0$ on every sample path for every $n\le k$.
\item Under $P_{\Gauss}$, for every $t>0$, with probability at least $1-e^{-t^2/2}$,
\begin{equation}\label{eq:gauss-lower}
  \inf_{u\in A_{p,k}}\|X_{\Gauss}u\|_2
  \ge\sqrt{n-1}-w(A_{p,k})-t.
\end{equation}
\end{enumerate}
\end{theorem}

The heavy-tailed law is $Z_{\HT}=\sigma\rho b_J$, where $J$ is uniform on $[M]$ and $\rho=L/\sqrt{\E L^2}$.  Tightness gives isotropy and the slab law gives global small-ball.  Given a sample, the set $D$ of distinct row types has size at most $n$, and $u_D$ annihilates every row.  Conversely, \eqref{eq:gauss-lower} is Gordon's escape theorem.  The cap-to-cone argument gives \eqref{eq:isotropic-width} for the full cone rather than only its generators.  The complete proof is in \Cref{app:isotropic}.

\begin{corollary}[Constant-width separation on one descent cone]\label[corollary]{cor:constant-isotropic}
Universal constants can be chosen so that the following holds.  For every sufficiently large $p$, there is a descent-cone section $A_p\subset\Sph^{p-1}$ with $w(A_p)\le C_G$ and two centered isotropic designs with common absolute small-ball constants.  For every $0<\delta<1$, Gaussian measurements satisfy $\Lam_n(X_{\Gauss};A_p)\ge c_G$ with probability at least $1-\delta$ whenever $n\ge C_G(1+\log(1/\delta))$, whereas the heavy-tailed design has $\Lam_n(X_{\HT};A_p)=0$ on every sample path for every $n\le c_H\sqrt{p/\log p}$.
\end{corollary}

Choose $k=\lfloor a\sqrt{p/\log p}\rfloor$ with the sufficiently small universal constant $a>0$ fixed in \Cref{app:isotropic}.  The two designs have the same covariance and the same marginal small-ball constants.  What differs is how visibility events are coupled across directions.  Alternatively, $k=\lfloor p^{1/3}\rfloor$ gives $w(A_{p,k})\to0$ while the deterministic failure horizon diverges.

\subsection{Smooth-density robustness}

The isotropic counterexample above is discrete in its row type.  To show that the obstruction is not caused by atoms, smooth it with an independent Gaussian.  For $H\sim\Normal(0,I_p)$ and $\tau>0$, set
\begin{equation}\label{eq:smooth}
  Z_{\HT,\tau}=\frac{Z_{\HT}+\tau H}{\sqrt{1+\tau^2}}.
\end{equation}
The normalization preserves isotropy, while Gaussian convolution gives an everywhere-positive $C^\infty$ density.  Smoothing removes the exact kernel, but the next result shows that the restricted energy remains of order $\tau^2$.

\Needspace{8\baselineskip}
\begin{corollary}[Smooth isotropic robustness]\label[corollary]{cor:smooth}
Under the construction of \Cref{thm:isotropic}, there are absolute $\tau_0,\bar\alpha,\bar\beta,c_s>0$ such that, for every sufficiently large $p$, every admissible $1\le k\le cp/\log p$, and every $0<\tau\le\tau_0$, the law of $Z_{\HT,\tau}$ is centered and isotropic, has finite moments of every order and no exponential moment in any nonzero direction, has an everywhere-positive $C^\infty$ density on $\R^p$, and satisfies the global small-ball condition $(\bar\alpha,\bar\beta)$.  For each fixed integer $1\le n\le k$,
\begin{equation}\label{eq:smooth-upper}
  \Pp\{\Lam_n(X_{\HT,\tau};A_{p,k})\le2\tau^2\}\ge1-e^{-c_sn}.
\end{equation}
\end{corollary}

Conditionally on the row types, the adaptive witness $u_D$ is independent of the Gaussian perturbations, and its normalized empirical energy is $(\tau^2/(1+\tau^2))(\chi_n^2/n)$.  A standard chi-square bound gives the stated probability.  The witness is no longer an exact kernel direction, but $\sqrt{\Lam_n}$ can be arbitrarily small while the small-ball constants stay fixed.

\subsection{A distribution-free isotropic fallback}\label{sec:radius}

For bounded nonempty $A$, let
$r(A)=\inf_a\sup_{u\in A}\|u-a\|_2$ and $q(A)=\dim\aff(A)$.

\Needspace{8\baselineskip}
\begin{theorem}[Dimension--radius RE bound]\label[theorem]{thm:radius}
Let $Z$ be centered and isotropic and let $\varnothing\ne A\subseteq\Sph^{p-1}$.  Suppose $\SB_\alpha(P,A)\ge\beta$ for some $\alpha>0$ and $0<\beta\le1$.  For every $t>0$, with probability at least $1-e^{-t^2/2}$,
\begin{equation}\label{eq:radius-bound}
  \inf_{u\in A}\|Xu\|_2
  \ge\frac{\alpha\beta}{2}\sqrt n-2r(A)\sqrt{q(A)}-\frac{\alpha t}{2}.
\end{equation}
\end{theorem}

\begin{corollary}[Sample-complexity consequence]\label[corollary]{cor:radius-sample}
Under the assumptions of \Cref{thm:radius}, for every $0<\delta<1$ there is a universal $C$ such that
\begin{equation}\label{eq:radius-sample}
  n\ge C\left(\frac{q(A)r(A)^2}{\alpha^2\beta^2}
     +\frac{\log(1/\delta)}{\beta^2}\right)
  \quad\Longrightarrow\quad
  \Lam_n(X;A)\ge\frac{\alpha^2\beta^2}{16}
\end{equation}
with probability at least $1-\delta$.
\end{corollary}

The proof of \Cref{thm:radius} is short.  For $H_n=n^{-1/2}\sum_i\eps_i Z_i$, translate by a minimum enclosing center $a$ and project onto $L=\spanop(A-a)$.  Isotropy gives
\begin{equation}\label{eq:emp-width-radius}
  \E\sup_{u\in A}\langle H_n,u\rangle
  \le r(A)\E\|P_L H_n\|_2\le r(A)\sqrt{q(A)}.
\end{equation}
Substitution into the empirical small-ball inequality proves \eqref{eq:radius-bound}.  For $A_{p,k}$, \Cref{app:isotropic} proves $q(A_{p,k})=p$ and $r(A_{p,k})^2\asymp k/p$.  At fixed confidence and fixed small-ball constants, the worst-case sample complexity on this family is therefore $\Theta(k)$.  This is a family-wise sharpness statement, not a joint minimax characterization for every prescribed pair of threshold and radius complexities.
Thus isotropy does not recover Gaussian width, but it prevents arbitrarily bad behavior once the set is controlled by its affine dimension and enclosing radius.

\section{Consequences and scope}\label{sec:consequences}

Our polyhedral descent-cone sections are closed, so their spherical sections are compact.  Whenever $\Lam_n(X;A)=0$, the minimum is attained by some $u\in A\cap\ker X$.  Consequently the noiseless program
\begin{equation}\label{eq:convex-program}
  \min_\theta\mathcal R(\theta)
  \quad\text{subject to}\quad X\theta=X\theta^\star
\end{equation}
does not uniquely recover $\theta^\star$.  The explicit and isotropic counterexamples therefore translate directly into failure of norm-regularized recovery.

For the smoothed construction, write $\kappa_n(X;A)=\sqrt{\Lam_n(X;A)}$ for the normalized restricted singular value and define the realized-design inverse modulus $K_n(X;A)=1/\kappa_n(X;A)$, with $K_n=\infty$ when $\kappa_n=0$.  On the event in \Cref{cor:smooth}, $K_n\ge1/(\sqrt2\tau)$ and the corresponding quadratic modulus $1/\Lam_n$ is at least $1/(2\tau^2)$.  This is a deterministic conditioning statement for each realized design; its witness may depend on $X$, so it is not a two-fixed-parameter statistical minimax claim.  The obstruction is nevertheless relevant to noisy recovery, not only exact interpolation.

The two positive mechanisms in the paper are complementary.  Relative threshold occupancy is distribution-adapted and requires no moments.  The dimension--radius bound is Euclidean but uses isotropy.  Stronger increment or moment assumptions can recover finer chaining and, in suitable models, Gaussian-width behavior.  We do not claim that threshold VC complexity and dimension--radius complexity form a joint minimax formula for every distribution class.

\section{Related work}\label{sec:related}

\paragraph{Restricted eigenvalues and Gaussian geometry.}
RE theory underpins the Lasso, Dantzig selector, and decomposable regularizers \citep{bickel2009simultaneous,negahban2012unified}.  For Gaussian measurements, escape through a mesh, Gaussian width, and statistical dimension connect descent cones to sharp recovery transitions \citep{gordon1988escape,chandrasekaran2012convex,amelunxen2014living,tropp2015convex}.  Related bounds cover sub-Gaussian, correlated, anisotropic, and stable-rank models \citep{banerjee2014estimation,raskutti2010restricted,rudelson2013reconstruction,kasiviswanathan2018restricted}; broader universality laws still require entry or tail structure beyond marginal visibility \citep{oymak2018universality}.  All constrain the joint process indexed by $A$.

\paragraph{Small-ball methods and distribution-dependent complexity.}
The small-ball method lower-bounds nonnegative empirical processes without upper-tail concentration \citep{koltchinskii2015bounding,mendelson2015learning,tropp2015convex}.  Its distribution-dependent Rademacher or localized complexities become quadratic and multiplier fixed points in regularized estimation \citep{lecue2018regularization,lecue2017regularizationII}.  Extensions relax uniform small-ball assumptions \citep{mendelson2021extending}, and moment-based covariance bounds offer another route \citep{oliveira2016lower}.  These retain information absent from ordinary Gaussian width.

\paragraph{Positive results under stronger tail assumptions.}
Weak coordinate moments can suffice for sparse recovery \citep{lecue2017sparse,dirksen2018gap}.  Under a uniform $\psi_1$ assumption, \citet{sivakumar2015beyond} use exponential width and threshold VC dimension at scale $n\gtrsim d/\beta^2$.  Sub-Weibull bounds \citep{kuchibhotla2022moving}, mixed-tail chaining \citep{dirksen2015tail,genzel2022generic}, and thresholding \citep{wei2018structured} likewise add coupling or robustness that excludes our encoding.

\paragraph{Earlier sparse-recovery results and the 2015 open problem.}
First circulated in 2014, the spiky construction of \citet{lecue2017sparse} has centered isotropic rows with bounded fourth moments and hence a uniform small-ball condition.  For an $\ell_1$-related sparse-recovery geometry, it can force the relevant RE and compatibility quantities to vanish with constant probability when $n\log n\lesssim p$; the companion note \citep{lecue2014necessary} records a basis-pursuit failure.  Thus this construction already gives a special-case counterexample to a fully uniform implication from marginal small-ball behavior to Gaussian-width-controlled RE.  The COLT note cited Lecué--Mendelson in its discussion of unit $s$-sparse directions and the Koltchinskii--Mendelson threshold-VC method, but explicitly characterized the available results as covering only ``certain special cases of $A$'' and left the question for general $A\subseteq\Sph^{p-1}$ under the small-ball property open \citep[p.~1754]{banerjee2015open}.  Our results go beyond that geometry-specific instance by establishing constant-width pathwise failure, exact realization of arbitrary finite threshold systems, descent-cone universality, the sharp fixed-$d$ occupancy law, and same-geometry isotropic and smooth separations.

\paragraph{Threshold classes and range spaces.}
The COLT note highlighted VC-controlled threshold bounds \citep{koltchinskii2015bounding,banerjee2015open}.  Classical random $\eps$-nets and relative approximations depend on VC dimension \citep{haussler1987epsilon,li2001improved,harpeled2011relative}; \citet{komlos1992almost} prove the logarithmic lower bound, and \citet{pach2013tight} give a geometric VC-two construction.  Those inequalities are not new here.  Our contribution is their exact realization on arbitrarily narrow spherical sets and a descent-cone lift turning every missed range into a kernel witness.

\section{Conclusion}

We show that the Gaussian-width principle posed in the COLT 2015 open-problem note does not follow from a uniform small-ball condition: it can fail pathwise on a constant-width descent cone.  Empirical RE requires one sample to cover all relevant directions, whereas bare small-ball visibility is only pointwise.  Gaussian and sub-Gaussian increment control couples nearby directions, but a marginal small-ball condition does not.  Encoding arbitrary threshold range spaces in narrow spherical caps yields pathwise counterexamples and the sharp low-mass law at each fixed threshold VC dimension.  The separation survives isotropy, moments of every order, and positive smooth densities, while isotropy still gives a sharp dimension--radius fallback.  A natural next question is which minimal structure between marginal visibility and full increment control restores a Gaussian-width law.

\clearpage
\bibliography{references}

\clearpage
\appendix
\crefalias{section}{appendix}
\numberwithin{equation}{section}
\input{appendix}

\end{document}

%% file: appendix.tex
\section{Auxiliary facts}\label{app:aux}

This appendix records all proofs and the external tools used in them.  \Cref{app:aux} collects the empirical small-ball inequality, Gordon's Gaussian escape bound, range-space estimates, and two geometric lemmas proved here.  \Cref{app:explicit} proves the explicit constant-width counterexample.  \Cref{app:exact-range,app:cone-range,app:sharp-vc} establish exact range-space representation, its descent-cone lifting, and the sharp threshold-complexity law.  \Cref{app:frame,app:isotropic} construct the anchored isotropic frame and prove the isotropic separation.  \Cref{app:smooth,app:radius,app:recovery} treat smoothing, the dimension--radius guarantee, and the consequences for convex recovery.  Throughout, numerical constants may change from one display to the next.

Unless stated otherwise, all auxiliary random objects introduced within a construction---including $L$, $\sigma$, $J$, $Y$, $H$, and their sample copies---are mutually independent.  Suprema over non-countable sets are interpreted in outer expectation or outer probability.  For the finite or compact polyhedral constructions used here, the relevant suprema are measurable and this distinction disappears.

When an index set lies in a Euclidean subspace $V$ of an ambient space, its Gaussian width is computed using a standard Gaussian vector on $V$.  Equivalently, one may use an ambient standard Gaussian and project it orthogonally onto $V$, since all relevant inner products are unchanged.

\subsection{Empirical small-ball and Gaussian escape inequalities}

For a set $A\subseteq\Sph^{p-1}$ and a distribution $P$ on $\R^p$, define
\[
  \SB_\xi(P,A)=\inf_{u\in A}\Pp\bigl(|\langle Z,u\rangle|\ge\xi\bigr),
  \qquad
  W_n(A;P)=\E\sup_{u\in A}\left\langle \frac1{\sqrt n}\sum_{i=1}^n\eps_i Z_i,u\right\rangle,
\]
where $Z,Z_1,\ldots,Z_n$ are distributed according to $P$ and the $\eps_i$ are independent Rademacher variables. We use the following form of the small-ball inequality, which is Proposition 5.1 of \citet{tropp2015convex}.

\begin{lemma}[Small-ball inequality]\label[lemma]{lem:small-ball-standard}
For every $\xi,t>0$, with probability at least $1-e^{-t^2/2}$,
\[
  \inf_{u\in A}\|Xu\|_2
  \ge
  \xi\sqrt n\,\SB_{2\xi}(P,A)-2W_n(A;P)-\xi t.
\]
\end{lemma}

For a standard Gaussian matrix $G\in\R^{n\times p}$, Gordon's escape theorem gives the companion inequality
\begin{equation}\label{eq:gordon-app}
  \Pp\left(
    \inf_{u\in A}\|Gu\|_2
    \ge \E\|g_n\|_2-w(A)-t
  \right)
  \ge1-e^{-t^2/2},
\end{equation}
where $g_n\sim\Normal(0,I_n)$. We use the elementary estimate $\E\|g_n\|_2\ge\sqrt{n-1}$.

\subsection{Two geometric lemmas}

\begin{lemma}[From a finite cap to its cone]\label[lemma]{lem:cap-cone}
Let $u_\star\in\Sph^{p-1}$ and let $\cU\subseteq\Sph^{p-1}$ be finite. Suppose
\[
  \max_{u\in\cU}\|u-u_\star\|_2\le r<1.
\]
Set $C=\cone(\cU)$ and $A=C\cap\Sph^{p-1}$. Then
\begin{equation}\label{eq:cap-cone-width}
  w(A)
  \le
  r\sqrt{2\log |\cU|}+\frac{r^2}{2}\sqrt p.
\end{equation}
Moreover, every $v\in A$ satisfies
\begin{equation}\label{eq:cap-cone-radius}
  \|v-u_\star\|_2\le r+r^2/2.
\end{equation}
\end{lemma}

\begin{proof}
Let $K=\conv(\cU)$. Every nonzero point in $C$ is a positive multiple of a point in $K$, hence each $v\in A$ can be written as $v=x/\|x\|_2$ for some $x\in K$. Since every $u\in\cU$ is a unit vector,
\[
  \langle u,u_\star\rangle
  =1-\frac12\|u-u_\star\|_2^2
  \ge1-r^2/2.
\]
The same inequality holds for $x\in K$. Also $\|x\|_2\le1$, and therefore
\[
  \left\|\frac{x}{\|x\|_2}-x\right\|_2
  =1-\|x\|_2
  \le1-\langle x,u_\star\rangle
  \le r^2/2.
\]
This proves \eqref{eq:cap-cone-radius}. For the width, write $u=u_\star+(u-u_\star)$ and apply the Gaussian maximal inequality to the centered Gaussian family $\{\langle g,u-u_\star\rangle:u\in\cU\}$:
\[
  w(K)=w(\cU)
  \le r\sqrt{2\log|\cU|}.
\]
For each $v=x/\|x\|_2$ as above,
\[
  \langle g,v\rangle
  \le \langle g,x\rangle+\frac{r^2}{2}\|g\|_2.
\]
Taking the supremum and expectation, and using $\E\|g\|_2\le\sqrt p$, gives \eqref{eq:cap-cone-width}.
\end{proof}

\begin{lemma}[Cone realization]\label[lemma]{lem:cone-norm}
Let $C\subset\R^p$ be a full-dimensional pointed polyhedral cone generated by $c_1,\ldots,c_N$. There is a centrally symmetric full-dimensional polytope $B$, a point $x^\star\in\partial B$, and a polyhedral norm $\cR$ whose unit ball is $B$ such that
\[
  \cD(\cR,x^\star)=C.
\]
The construction can be made explicit from the generators.
\end{lemma}

\begin{proof}
Choose $c_\circ\in\operatorname{int}(C)$.  Since the interior is open and is preserved by positive scaling, for each $i$ we have $c_\circ-c_i/(2L)\in\operatorname{int}(C)$ once $L$ is large enough.  Because there are finitely many generators, one may choose a common $L>0$ such that
\[
  2Lc_\circ-c_i\in\operatorname{int}(C)
  \qquad(i\in[N]).
\]
Set $x^\star=-Lc_\circ$ and
\[
  B=\conv\bigl(
      \{x^\star,x^\star+c_i:i\in[N]\}
      \cup
      \{-x^\star,-x^\star-c_i:i\in[N]\}
     \bigr).
\]
The set $B$ is centrally symmetric and full dimensional because the $c_i$ span $\R^p$. In particular, its center $0$ lies in its interior, so the Minkowski functional $\gamma_B$ is a norm. Every displacement from $x^\star$ to a listed point of $B$ is one of
\[
  c_i,\qquad 2Lc_\circ,\qquad 2Lc_\circ-c_i,
\]
and hence lies in $C$.  Conversely, every generator $c_i$ occurs, so
\begin{equation}\label{eq:radial-cone-identity}
  \cone(B-x^\star)
  =\cone\{c_i,2Lc_\circ,2Lc_\circ-c_i:i\in[N]\}
  =C.
\end{equation}
To verify that $x^\star$ is on the boundary, choose a linear functional $\ell\in\operatorname{int}(C^\ast)$.  Pointedness and full dimensionality imply
$\langle\ell,z\rangle>0$ for every nonzero $z\in C$.  Every other listed point $y$ satisfies $y-x^\star\in C\setminus\{0\}$, hence $\langle\ell,y\rangle>\langle\ell,x^\star\rangle$.  Thus $x^\star$ is an exposed vertex of $B$ and $\gamma_B(x^\star)=1$.

Finally, the definition of the gauge gives
\[
  \cD(\gamma_B,x^\star)
  =\{h:x^\star+th\in B\text{ for some }t>0\}
  =\cone(B-x^\star).
\]
The second equality holds in both directions by writing $h=t^{-1}(y-x^\star)$ or choosing $t$ as the reciprocal of the conic multiplier.  Combining this identity with \eqref{eq:radial-cone-identity} proves the claim with $\cR=\gamma_B$.
\end{proof}

\subsection{Range-space estimates}

We use two classical consequences of VC theory. First, if a class $\mathscr C$ has VC dimension at most $d$ and every $C\in\mathscr C$ has probability at least $\beta$, then
\begin{equation}\label{eq:relative-vc-app}
  n\ge
  C\frac{d\log(2/\beta)+\log(1/\delta)}{\beta}
\end{equation}
implies, with probability at least $1-\delta$,
\[
  \inf_{C\in\mathscr C}\mu_n(C)\ge\beta/2.
\]
This is a constant-relative-error specialization of relative approximation bounds \citep{li2001improved,harpeled2011relative}.

For a probability range space $(\Omega,\mu,\mathscr C)$, a finite set $S\subseteq\Omega$ is a $\beta$-net if it intersects every $C\in\mathscr C$ with $\mu(C)\ge\beta$.  Write $f_d(\beta)$ for the supremum, over finite probability range spaces of VC dimension at most $d$, of the minimum cardinality of a $\beta$-net.  Theorem~2.1 of \citet{komlos1992almost} implies
\begin{equation}\label{eq:kpw-limit}
  \liminf_{\beta\downarrow0}
  \frac{f_d(\beta)}{\beta^{-1}\log(1/\beta)}
  \ge d-2+\frac{2}{d+2}
  \qquad(d\ge2).
\end{equation}
The following finite form is the interface needed below.

\begin{lemma}[Finite $\beta$-net lower bound]\label[lemma]{lem:finite-net-lower}
There is a universal $c_{\mathrm{net}}>0$ such that, for every fixed $d\ge2$, there is $\beta_d>0$ with the following property.  For each $0<\beta\le\beta_d$, there is a finite full-support probability range space $(\Omega,\mu,\mathscr C)$ of VC dimension at most $d$ such that every set $S\subseteq\Omega$ with
\[
  |S|<c_{\mathrm{net}}\frac d\beta\log\frac1\beta
\]
misses some range $C\in\mathscr C$ of mass $\mu(C)\ge\beta$.
\end{lemma}

\begin{proof}
The coefficient on the right of \eqref{eq:kpw-limit} satisfies
\[
  \frac{d-2+2/(d+2)}{d}\ge\frac14
  \qquad(d\ge2).
\]
Hence, for every fixed $d$, the liminf statement gives the asserted order with a universal numerical constant once $\beta$ is below a threshold $\beta_d$ that may depend on $d$.  Because $f_d$ is a supremum, one chooses a finite range space whose minimum net size is, say, at least half the displayed lower bound; no attainment of the supremum is needed.  The constructions in the cited theorem use a finite ground set with the uniform measure, hence have full support.

Our convention uses $\mu(C)\ge\beta$.  If the strict convention $\mu(C)>\varepsilon$ is used in the cited formulation, apply it with $\varepsilon=2\beta$ and retain the ranges of mass greater than $2\beta$.  They all have mass at least $\beta$, their VC dimension cannot increase, and, after decreasing $\beta_d$,
\[
  \frac{d}{2\beta}\log\frac1{2\beta}
  \ge c\frac d\beta\log\frac1\beta.
\]
This changes only the universal constant and proves the stated weak-threshold form.
\end{proof}

The threshold $\beta_d$ is not asserted to be uniform in $d$.  \citet{pach2013tight} give a geometric VC-two construction with the same logarithmic phenomenon.  For VC dimension one, deterministic net size alone does not contain this logarithm, but random sampling does, by the coupon-collector argument in \Cref{app:sharp-vc}.

\section{Proof of the explicit counterexample}\label{app:explicit}

\begin{proof}

\subsection{The norm and its descent cone}

Recall
\[
  V_m=\left\{(t,y/m):t\in\R,\ y\in\R^m,\ \sum_{j=1}^m y_j=0\right\}
\]
and
\[
  \cR_m(t,y/m)
  =\max\left\{|t|,\max_j|t-y_j|,\max_j|t+y_j|\right\}.
\]
This is the maximum of finitely many absolute linear functionals. If it vanishes, then $t=0$ and $t\pm y_j=0$ for every $j$, hence $y=0$. It is therefore a norm on $V_m$.

Let $\theta_m^\star=(-1,0)$ and let $v=(t,y/m)$. For $\lambda>0$,
\[
 \cR_m(\theta_m^\star+\lambda v)
 =\max\left\{
 |-1+\lambda t|,
 \max_j|-1+\lambda(t-y_j)|,
 \max_j|-1+\lambda(t+y_j)|
 \right\}.
\]
There exists a sufficiently small $\lambda>0$ for which this maximum is at most one exactly when
\[
  t\ge0,\qquad t-y_j\ge0,\qquad t+y_j\ge0\quad(j\in[m]).
\]
Thus
\[
  \cD(\cR_m,\theta_m^\star)
  =C_m
  =\{(t,y/m):t\ge0,\ |y_j|\le t,\ \sum_j y_j=0\}.
\]

\subsection{Uniform small-ball bound}

For
\[
  r_j=e_0+me_j-\sum_{\ell=1}^m e_\ell
\]
and $v=(t,y/m)\in C_m$, direct calculation gives
\[
  \langle r_j,v\rangle=t+y_j.
\]
Set $a_j=t+y_j$. Then $0\le a_j\le2t$ and $m^{-1}\sum_j a_j=t$. If fewer than $m/3$ of the $a_j$ were at least $t/2$, then
\[
  \frac1m\sum_j a_j
  <\frac13(2t)+\frac23(t/2)=t,
\]
a contradiction. Also
\[
  \|v\|_2^2=t^2+\frac1{m^2}\sum_j y_j^2
  \le t^2+\frac1m t^2
  \le2t^2.
\]
Since $L\ge1$, for $Z_m=\sigma Lr_J$,
\[
  \Pp\left(
    |\langle Z_m,v\rangle|
    \ge\frac1{2\sqrt2}\|v\|_2
  \right)
  \ge\frac13.
\]
Homogeneity gives the claimed bound on $A_m$.

\subsection{Gaussian width}

For $v=(t,y/m)\in A_m$, the cone constraints and $\|v\|_2=1$ imply $0\le t\le1$ and $|y_j|\le t$.  Let $g=(g_0,g_1,\ldots,g_m)$ be standard Gaussian in the ambient $\R^{m+1}$.  Its orthogonal projection onto $V_m$ is a standard Gaussian vector on $V_m$, and inner products with $v\in V_m$ are unchanged.  Hence
\[
  \sup_{v\in A_m}\langle g,v\rangle
  \le |g_0|+\frac1m\sum_{j=1}^m|g_j|.
\]
Taking expectations yields
\[
  w(A_m)\le2\E|g_0|=2\sqrt{2/\pi}.
\]

\subsection{Deterministic kernel and VC dimension}

Let $D\subseteq[m]$ denote the distinct indices observed among $J_1,\ldots,J_n$. If $n\le m/2$, choose a set $N\subseteq[m]$ of cardinality $m/2$ with $D\subseteq N$, and define $s_j=-1$ on $N$ and $s_j=1$ off $N$. Then $s\in\{-1,1\}^m$ and $\sum_j s_j=0$. Therefore
\[
  u_s=\frac{(1,s/m)}{\sqrt{1+1/m}}\in A_m.
\]
For each row type,
\[
  \langle r_j,u_s\rangle
  =\frac{1+s_j}{\sqrt{1+1/m}}.
\]
All observed rows have indices in $D\subseteq N$, hence $Xu_s=0$. This proves \eqref{eq:explicit-zero} for every realization.

To prove the VC lower bound, fix any $I\subseteq[m]$ of cardinality $m/2$ and choose one support point $z_j=\ell_0r_j$ for each $j\in I$, where $\ell_0>1$ belongs to the support of $L$. Given any label set $B\subseteq I$, prescribe $s_j=1$ for $j\in B$ and $s_j=-1$ for $j\in I\setminus B$. Among the remaining $m/2$ coordinates, choose signs so that the total sum is zero. This is possible because the partial sum on $I$ has the same parity as $m/2$ and lies between $-m/2$ and $m/2$. For the resulting balanced $s$,
\[
  \ind\{|\langle z_j,u_s\rangle|\ge1\}
  =\ind\{s_j=1\}
  \qquad(j\in I),
\]
since $2\ell_0/\sqrt{1+1/m}>1$. Thus $I$ is shattered and $\VC(\cT_1(A_m)|_{\operatorname{supp}P_m})\ge m/2$.

The random vector is centered because of $\sigma$ and has finite moments of every order because $L$ does.  To verify the spanning claim, observe that $m^{-1}\sum_jr_j=e_0$ and $r_j-r_m=m(e_j-e_m)$ for $j<m$.  These vectors span the anchor coordinate and the zero-sum subspace, hence all of $V_m$.  Therefore, for every nonzero $v\in V_m$, some $\langle r_j,v\rangle$ is nonzero. Conditioning on $J=j$ shows $\E e^{s|\langle Z_m,v\rangle|}=\infty$ for all $s>0$.
\end{proof}

\section{Exact range-space representation}\label{app:exact-range}

\begin{proof}
Let $A=\{u_1,\ldots,u_M\}$ and $Z$ be as in \Cref{thm:exact-range}. The independent sign centers $Z$. Since $L>1$ almost surely,
\[
  \langle Z,u_j\rangle
  =\sigma L\ind\{Y\in C_j\}
\]
implies the claimed small-ball inequality and \eqref{eq:range-exact}.  To justify the VC equality carefully, map a support point $z$ to the incidence vector
$I(z)=(\ind\{|\langle z,u_j\rangle|\ge1\})_{j\le M}$.  The displayed identity says that the set of incidence vectors on $\operatorname{supp}(P)$ equals the set of incidence vectors $(\ind\{y\in C_j\})_{j\le M}$ on $\operatorname{supp}(\mu)$.  By the full-support convention in the theorem, $\operatorname{supp}(\mu)=\Omega$.  Repeated latent points, signs, and radial values only duplicate an incidence vector.  Choosing one representative from each nonempty incidence fiber transfers shattered sets in both directions, proving equality of the support-restricted VC dimensions.

For each sample and each $j$,
\[
  \min\{1,\langle Z_i,u_j\rangle^2\}
  =\ind\{Y_i\in C_j\},
\]
which proves \eqref{eq:clipped-exact}. The untruncated sum in direction $u_j$ is positive exactly when some $Y_i$ belongs to $C_j$. Taking the minimum over $j$ proves the hitting equivalence.

Finally,
\[
  \sup_{j\le M}\langle g,u_j\rangle
  =\frac{g_0+\eps\max_{j\le M}g_j}{\sqrt{1+\eps^2}}.
\]
The expectation of $g_0$ is zero and $\E\max_j g_j\le\sqrt{2\log(2M)}$.  Hence the claimed width bound follows whenever
\[
  \eps\le\frac{\eta}{\sqrt{2\log(2M)}}.
\]
This formula also covers $M=1$ and shows that the width can be sent to zero without altering the represented range system.
\end{proof}

\section{Descent-cone range universality}\label{app:cone-range}

\begin{proof}
The map $T_\eps$ in \eqref{eq:cone-parametrization} is injective, so
\[
  \cR_\eps(T_\eps a)=\|a\|_\infty
\]
defines a norm on $V_\eps$.  Its unit ball is $T_\eps[-1,1]^M$, a centrally symmetric polytope in $V_\eps$; hence $\cR_\eps$ is a polyhedral norm.  For $h=T_\eps a$,
\[
  \cR_\eps(-T_\eps\ind+th)\le1
  \quad\Longleftrightarrow\quad
  \|-\ind+ta\|_\infty\le1.
\]
The latter holds for some $t>0$ if and only if $a\ge0$ coordinatewise. This proves the cone identity in \eqref{eq:cone-parametrization}. Every nonzero $a\ge0$ can be written as $(\ind^\top a)q$ with $q\in\Delta_M$, and normalization gives the parametrization in \eqref{eq:cone-parametrization}.

Let
\[
  \widetilde Z
  =\frac{2\sqrt{1+\eps^2}}{\beta\eps}
    \sigma L\sum_{j=1}^M\ind\{Y\in C_j\}e_j
  \in\R^{M+1},
\]
and let $Z$ be its orthogonal projection onto $V_\eps$. Inner products with $u(q)\in V_\eps$ are unchanged, and
\begin{align*}
  |\langle Z,u(q)\rangle|
  &=\frac{2\sqrt{1+\eps^2}}{
      \beta\sqrt{1+\eps^2\|q\|_2^2}}
      L f_q(Y)\\
  &\ge \frac2\beta f_q(Y).
\end{align*}
Because $0\le f_q\le1$ and $\E f_q\ge\beta$, if
$p_q=\Pp(f_q\ge\beta/2)$, then
\[
  \beta\le\E f_q
  \le p_q+(1-p_q)\beta/2,
\]
and therefore $p_q\ge\beta/(2-\beta)$. This proves the uniform small-ball bound.

For the width, let $g=(g_0,g')$ be standard Gaussian in $\R^{M+1}$. Put $d_q=\sqrt{1+\eps^2\|q\|_2^2}$. If $g_0\ge0$, then $g_0/d_q\le g_0$. If $g_0<0$, then
\[
  \frac{g_0}{d_q}
  =g_0+|g_0|\left(1-\frac1{d_q}\right)
  \le g_0+\frac{\eps^2}{2}|g_0|.
\]
Also
\[
  \frac{\eps\langle g',q\rangle}{d_q}
  \le\eps\max\{0,g_1,\ldots,g_M\}.
\]
Consequently,
\[
  w(A_\eps)
  \le\eps\E\max\{0,g_1,\ldots,g_M\}
      +\frac{\eps^2}{2}\E|g_0|
  \le\eps\sqrt{2\log(2M)}+\frac{\eps^2}{\sqrt{2\pi}}.
\]
Choosing $\eps$ small gives any prescribed $\eta$. Finally, if the latent sample misses $C_j$, then $f_{e_j}(Y_i)=0$ for all $i$, hence $Xu(e_j)=0$.
\end{proof}

\section{Sharp threshold-complexity law}\label{app:sharp-vc}

\begin{proof}
The definition of $\NRE$ asks for a threshold valid for every larger sample size.  If the universal conclusion fails at some $N_0$, then no $n\le N_0$ satisfies the definition, and hence $\NRE\ge N_0+1$.  Each upper bound below applies marginally at every larger sample size.

\paragraph{Upper bound.}
Apply the relative VC estimate \eqref{eq:relative-vc-app} to the support-restricted threshold class $\cT_1(A)$.  For every
$N\ge C[d\log(2/\beta)+\log(1/\delta)]/\beta$, with probability at least $1-\delta$,
$\widehat q_{N,1}(A)\ge\beta/2$.  The deterministic inequality \eqref{eq:occ-suff} then gives $\Lam_N(X;A)\ge\beta/2$.

\paragraph{A universal $d/\beta$ lower bound.}
Assume first that $0<\beta\le1/4$.  Let $N_0=\lfloor d/\beta\rfloor$ and put the uniform measure on $[N_0]$.  Take as ranges all $d$-element subsets.  Each range has mass $d/N_0\ge\beta$, and the class has VC dimension exactly $d$: it shatters a fixed $d$-set, using points outside that set to complete each trace to cardinality $d$, and no $(d+1)$-set can be shattered.  Any hitting set must have at least $N_0-d+1$ points, since the complement of a smaller set contains a $d$-range.  Therefore every sample with fewer than $N_0-d+1\ge c d/\beta$ distinct points misses a range.  The exact representation theorem converts this deterministic miss into $\Lam_n=0$ while preserving threshold VC dimension.  This proves $\NRE\ge c d/\beta$.

For reference, the preceding $d/\beta$ construction has no hidden endpoint loss: with $n_0=N_0-d$, every sample of size $n_0$ misses a range, so
\[
  \NRE(d,\beta,\delta)\ge n_0+1=N_0-d+1
  \ge \frac34\frac d\beta.
\]
Here $\lfloor d/\beta\rfloor+1\ge d/\beta$ and $\beta\le1/4$ give the last inequality.

\paragraph{The coupon-collector term, including $d=1$.}
For $0<\beta\le1/16$, let $M=\lfloor1/(2\beta)\rfloor$ points each have mass $\beta$, put the remaining mass on one extra point, and let the ranges be the $M$ singletons.  This class has VC dimension one.  Let $U$ be the number of singleton ranges missed by $n$ independent observations.  Then
\begin{equation}\label{eq:coupon-moments}
  \E U=M(1-\beta)^n,\qquad
  \operatorname{Var}(U)\le\E U,
\end{equation}
because two missing indicators have covariance
$(1-2\beta)^n-(1-\beta)^{2n}\le0$.  If
$n\le(4\beta)^{-1}\log M$, then $(1-\beta)^n\ge e^{-2\beta n}\ge M^{-1/2}$, so $\E U\ge\sqrt M$.  Chebyshev's inequality yields
$P\{U=0\}\le\operatorname{Var}(U)/(\E U)^2\le M^{-1/2}$.  Here $M\ge8$, so the probability of hitting every singleton is strictly below $3/4$.  Moreover, $M\ge1/(4\beta)$ and hence $\log M\ge c\log(2/\beta)$ on $0<\beta\le1/16$.  The exact representation therefore gives the claimed order on this interval.  For $1/16<\beta\le1/4$, a single range of mass $\beta$ is missed at sample size one with probability $1-\beta\ge3/4$, while $\beta^{-1}\log(2/\beta)\le16\log32$.  Decreasing the same universal constant covers this remaining interval and yields
\begin{equation}\label{eq:coupon-lower}
  \NRE(d,\beta,1/4)\ge c\,\frac{\log(2/\beta)}{\beta}
  \qquad(d\ge1).
\end{equation}
Indeed, on the first interval one may take $n_0=\lfloor(4\beta)^{-1}\log M\rfloor$; the probability of success at $n_0$ is below $3/4$, whence $\NRE\ge n_0+1>(4\beta)^{-1}\log M$.  This makes the integer reduction explicit.
This is why the random-sample problem retains a logarithm for VC dimension one, even though a smallest deterministic net need not.
Since $\NRE$ is nonincreasing as $\delta$ increases, \eqref{eq:coupon-lower} also holds for every $\delta\le1/4$.

\paragraph{The confidence term.}
An explicit range space is $\Omega=\{a,b\}$ with $\mu(a)=\beta$ and $\mathscr C=\{\{a\}\}$; it has full support and VC dimension zero.
Take one range of mass exactly $\beta$.  It is missed with probability $(1-\beta)^n\ge e^{-2\beta n}$ for $\beta\le1/4$.  Consequently, every integer
$n<(2\beta)^{-1}\log(1/\delta)$ has failure probability strictly larger than $\delta$.  Accounting for the integer endpoint only changes a universal constant, and therefore
\begin{equation}\label{eq:confidence-lower}
  \NRE(d,\beta,\delta)
  \ge c\frac{1}{\beta}\log\frac1\delta.
\end{equation}
More precisely, set $T=(2\beta)^{-1}\log(1/\delta)$ and $n_0=\lceil T\rceil-1<T$.  Failure at $n_0$ gives $\NRE\ge n_0+1=\lceil T\rceil\ge T$.
The maximum of the three lower bounds above is at least one third of their sum.  This proves the left side of \eqref{eq:sharp-vc-general}.

\paragraph{Sharp fixed-$d$ low-mass regime.}
Fix $d\ge2$ and take the finite full-support range space supplied by \Cref{lem:finite-net-lower}; decrease $\beta_d$ if necessary so that $\beta_d\le1/4$.  Removing any ranges of mass below $\beta$ cannot increase VC dimension.  By the lemma, every subset of $\Omega$ with fewer than $c_{\mathrm{net}}(d/\beta)\log(1/\beta)$ points misses a retained range of mass at least $\beta$.  The distinct support of a multiset sample has cardinality at most the sample size, so every sample path below this threshold misses such a range.  Applying \Cref{thm:exact-range} gives $\Lam_n=0$ on every path while preserving the support-restricted VC dimension.  For $d=1$, \eqref{eq:coupon-lower} gives the same order.  Finally, $\log(1/\beta)\asymp\log(2/\beta)$ on this range.  Combining this term with \eqref{eq:confidence-lower} proves \eqref{eq:sharp-vc} for every fixed $d$ and all $0<\beta\le\beta_d$.
For the strict inequality in \Cref{lem:finite-net-lower}, put $T_\beta=c_{\mathrm{net}}(d/\beta)\log(1/\beta)$ and $n_0=\lceil T_\beta\rceil-1<T_\beta$.  Every size-$n_0$ sample then fails pathwise, so $\NRE\ge n_0+1=\lceil T_\beta\rceil\ge T_\beta$.
\end{proof}

\section{The anchored isotropic frame}\label{app:frame}

\begin{proof}
We use the probabilistic method.  The proof has four steps.  We first control the empirical covariance, then establish a uniform slab lower bound, and next condition every row submatrix of size $O(p/\log p)$.  A final whitening step makes the frame exactly tight while preserving the common anchor and the three preceding estimates.  Let $M=p^4$ and let
\[
  a_j=(1,\xi_j)\in\R\times\R^{p-1},
  \qquad j\in[M],
\]
where the $\xi_j$ are independent vectors with independent Rademacher coordinates. Define
\[
  \Sigma=\frac1M\sum_{j=1}^M a_j a_j^\top.
\]
We establish three simultaneous events of positive probability.

\subsection{Covariance control}

The vectors $a_j$ are isotropic in the sense that $\E a_j a_j^\top=I_p$ and satisfy $\|a_j\|_2^2=p$ deterministically. Matrix Bernstein, applied to $a_j a_j^\top-I_p$, yields
\begin{equation}\label{eq:sigma-concentration}
  \Pp\left(\|\Sigma-I_p\|_{\mathrm{op}}>1/4\right)
  \le2p\exp(-cM/p).
\end{equation}
For completeness, if $Y_j=a_j a_j^\top-I_p$, then $\|Y_j\|_{\mathrm{op}}\le p+1\le2p$ and
\[
  \E Y_j^2=(p-1)I_p.
\]
The self-adjoint matrix Bernstein inequality \citep[Theorem~6.1.1]{tropp2015matrix}, applied at deviation level $M/4$, has variance proxy $M(p-1)$ and range bound $2p$.  Substitution gives an exponent at most $-cM/p$, proving the displayed estimate.  In particular, on the complementary event
\begin{equation}\label{eq:sigma-comparable}
  \frac34I_p\preceq\Sigma\preceq\frac54I_p.
\end{equation}

\subsection{A uniform empirical slab bound}

For fixed $v=(v_0,v')\in\Sph^{p-1}$ and $a=(1,\xi)$, put $S=\langle a,v\rangle=v_0+\langle\xi,v'\rangle$. Then
\[
  \E S^2=1
\]
and a direct fourth-moment expansion gives
\[
  \E S^4
  =v_0^4+6v_0^2\|v'\|_2^2
    +3\|v'\|_2^4-2\sum_{\ell=1}^{p-1}v_\ell^4
  \le3.
\]
Paley--Zygmund applied to $S^2$ gives
\begin{equation}\label{eq:population-slab}
  \Pp\bigl(|\langle a,v\rangle|\ge1/\sqrt2\bigr)
  \ge1/12.
\end{equation}
The class of two-sided homogeneous slabs
\[
  \bigl\{x:|\langle x,v\rangle|\ge\|v\|_2/\sqrt2\bigr\},
  \qquad v\ne0,
\]
has VC dimension at most $C_0p$.  Indeed, it is a subclass of unions of two affine halfspaces.  On $N$ points the halfspace growth function is at most $(eN/(p+1))^{p+1}$ by Sauer's lemma, so the growth function of two-fold unions is at most its square.  For a sufficiently large numerical $C_0$, this is smaller than $2^N$ whenever $N>C_0p$, proving the claim.  We use the following explicit form of the VC uniform law \citep{boucheron2013concentration}.  For a class of sets with VC dimension $v$ and $M\ge v$, symmetrization, Hoeffding's inequality, and Sauer's lemma give
\[
  \Pp\left\{\sup_C|\mu_M(C)-\mu(C)|>r\right\}
  \le8\left(\frac{2eM}{v}\right)^v e^{-Mr^2/32}.
\]
Take $v=C_0p$ and $r=1/24$.  Since $M=p^4$, the right side is at most $e^{-cM}$ for all sufficiently large $p$.  Combining this event with \eqref{eq:population-slab} gives
\begin{equation}\label{eq:empirical-slab}
  \inf_{v\ne0}\frac1M
  \left|\left\{j:
    |\langle a_j,v\rangle|\ge\|v\|_2/\sqrt2
  \right\}\right|
  \ge1/24.
\end{equation}

\subsection{All small row submatrices are well conditioned}

Fix $D\subseteq[M]$ with $s=|D|$ and write $\Xi_D\in\R^{(p-1)\times s}$ for the matrix whose columns are $\xi_j$, $j\in D$. For fixed $x\in\Sph^{s-1}$, the coordinates $Y_\ell=\sum_{j\in D}x_j\xi_{j\ell}$ of $\Xi_Dx$ are independent, mean-zero, and variance-one.  Hoeffding's lemma gives $\|Y_\ell\|_{\psi_2}\le C\|x\|_2=C$, so $\|Y_\ell^2-1\|_{\psi_1}\le C'$.  Bernstein's inequality for these centered squared values gives
\begin{equation}\label{eq:fixed-x-ri}
  \Pp\left(
     \left|\|\Xi_D x\|_2^2-(p-1)\right|>(p-1)/2
  \right)
  \le2e^{-cp}.
\end{equation}
A $1/8$-net $\mathcal N$ of $\Sph^{s-1}$ has cardinality at most $17^s$ \citep{vershynin2018high}.  To spell out the net step, put
$H=\Xi_D^\top\Xi_D-(p-1)I_s$.  For every symmetric $H$ and every $\varepsilon$-net with $\varepsilon<1/2$,
\[
  \|H\|_{\mathrm{op}}
  \le\frac{1}{1-2\varepsilon}\max_{x\in\mathcal N}|x^\top Hx|.
\]
This follows by choosing, for a maximizing unit vector $y$, a net point $x$ with $\|x-y\|\le\varepsilon$ and bounding
$|y^\top Hy-x^\top Hx|\le2\varepsilon\|H\|_{\mathrm{op}}$.
Apply \eqref{eq:fixed-x-ri} to every $x\in\mathcal N$.  On the resulting event, the last display with $\varepsilon=1/8$ puts all eigenvalues of $\Xi_D^\top\Xi_D$ between absolute multiples of $p$.  Hence
\begin{equation}\label{eq:fixed-D-ri}
  \Pp\left(
    cpI_s\npreceq\Xi_D^\top\Xi_D
    \text{ or }
    \Xi_D^\top\Xi_D\npreceq CpI_s
  \right)
  \le2\exp(-c'p+Cs).
\end{equation}
There are at most $(eM/s)^s$ subsets of cardinality $s$.  Thus the total failure probability for all sets of a fixed size $s$ is at most
\[
  2\exp\{-c'p+Cs+s\log(eM/s)\}.
\]
For $M=p^4$ and $s\le s_{\max}=\lfloor c_0p/\log p\rfloor$, the two positive terms are at most $6c_0p$ for large $p$.  Choose $c_0<c'/12$.  The failure probability for each size $s\le s_{\max}$ is then at most $2e^{-c'p/2}$, and summing over the at most $s_{\max}\le p$ possible sizes gives
\[
  2s_{\max}e^{-c'p/2}\le e^{-c''p}
\]
for all sufficiently large $p$.  Thus, with probability at least $1-e^{-c''p}$,
\begin{equation}\label{eq:all-D-xi}
  cpI_{|D|}\preceq\Xi_D^\top\Xi_D\preceq CpI_{|D|}
  \qquad
  \text{for every }|D|\le c_0p/\log p,
\end{equation}

Let $A_D$ be the matrix with rows $a_j^\top$, $j\in D$. For $x\in\R^D$,
\[
  \|A_D^\top x\|_2^2
  =(\ind^\top x)^2+\|\Xi_D x\|_2^2.
\]
The lower inequality in \eqref{eq:all-D-xi} is unchanged, while
$(\ind^\top x)^2\le s\|x\|_2^2\le c_0p\|x\|_2^2/\log p$.  Thus \eqref{eq:all-D-xi} implies
\begin{equation}\label{eq:all-D-a}
  cpI_{|D|}\preceq A_D A_D^\top\preceq CpI_{|D|}
\end{equation}
for all such $D$.

\subsection{Whitening}

The covariance, slab, and restricted-Gram failure probabilities are at most $2pe^{-cp^3}$, $e^{-cp^4}$, and $e^{-c''p}$, respectively.  Their sum is below one for all sufficiently large $p$.  Fix a realization on the three events and define
\[
  b_j=\Sigma^{-1/2}a_j,
  \qquad
  u_\star=\Sigma^{1/2}e_0.
\]
Then
\[
  \frac1M\sum_j b_j b_j^\top=I_p.
\]
Also
\[
  \|u_\star\|_2^2=e_0^\top\Sigma e_0=1,
  \qquad
  \langle b_j,u_\star\rangle
  =\langle a_j,e_0\rangle=1.
\]
For a unit vector $v$, set $q=\Sigma^{-1/2}v$. By \eqref{eq:sigma-comparable}, $\|q\|_2\ge\sqrt{4/5}$. Applying \eqref{eq:empirical-slab} to $q$ gives at least $M/24$ indices with
\[
  |\langle b_j,v\rangle|
  =|\langle a_j,q\rangle|
  \ge\|q\|_2/\sqrt2
  \ge\sqrt{2/5}.
\]
Thus one may take $a_0=1/2$ and $b_0=1/24$. Finally,
\[
  B_D B_D^\top=A_D\Sigma^{-1}A_D^\top,
\]
and \eqref{eq:sigma-comparable}, \eqref{eq:all-D-a} give \eqref{eq:gram}.
\end{proof}

\section{The isotropic separation}\label{app:isotropic}

\begin{proof}
Fix the frame from \Cref{lem:frame}, choose the theorem constant $c\le c_0$, and assume $1\le k\le cp/\log p$.  Throughout this proof, $D\subseteq[M]$ ranges only over $|D|\le k$, and $B_D$ denotes the corresponding row submatrix.  If $D\ne\varnothing$, \eqref{eq:gram} makes $G_D=B_DB_D^\top$ positive definite, and we set
\[
  P_D=B_D^\top G_D^{-1}B_D.
\]
For $D=\varnothing$, set $P_D=0$ and $s_D=0$.  In either case $P_D$ is the orthogonal projector onto the row span of $B_D$.  We first quantify the normalized residual $u_D$, then control its generated cone, and finally verify the distributional claims.  Since $B_Du_\star=\ind$ for nonempty $D$,
\begin{equation}\label{eq:projection-size}
  \|P_D u_\star\|_2^2
  =\ind^\top G_D^{-1}\ind.
\end{equation}
The Gram bounds imply, for $d=|D|\ge1$,
\begin{equation}\label{eq:projection-comparable}
  \frac{d}{C_1p}
  \le s_D^2\defeq\|P_D u_\star\|_2^2
  =\ind^\top G_D^{-1}\ind
  \le\frac{d}{c_1p}.
\end{equation}
For $D=\varnothing$, the same bounds hold with $d=s_D=0$.  For all sufficiently large $p$, the upper bound is at most $c/(c_1\log p)\le1/2$.  Thus $(I-P_D)u_\star\ne0$ and every $u_D$ used in the construction is well defined.

Since $(I-P_D)u_\star$ is orthogonal to $P_D u_\star$,
\[
  \langle u_D,u_\star\rangle=\sqrt{1-s_D^2}
\]
and
\begin{equation}\label{eq:distance-exact}
  \|u_D-u_\star\|_2^2
  =2\left(1-\sqrt{1-s_D^2}\right)
  =\frac{2s_D^2}{1+\sqrt{1-s_D^2}}.
\end{equation}
In particular, $s_D^2\le\|u_D-u_\star\|_2^2\le2s_D^2$ when $s_D^2\le1/2$, so \eqref{eq:projection-comparable} gives $\|u_D-u_\star\|_2^2\asymp d/p$ with explicit absolute constants.  Also $B_D u_D=0$.

\subsection{Geometry and width of the cone}

Let
\[
  \cU_{p,k}=\{u_D:|D|\le k\}.
\]
Its cardinality satisfies
\begin{equation}\label{eq:number-generators}
  |\cU_{p,k}|
  \le\sum_{d=0}^k\binom Md
  \le2\left(\frac{eM}{k}\right)^k
\end{equation}
for $1\le k\le M/2$. By \eqref{eq:distance-exact}, all generators lie in a cap of radius
\[
  r\le C\sqrt{k/p}
\]
about $u_\star$.  Since $k\le cp/\log p$, choosing $c$ small and $p$ large ensures $r<1$, as required by \Cref{lem:cap-cone}.  Applying that lemma with $M=p^4$ gives
\begin{align*}
  w(A_{p,k})
  &\le C\sqrt{\frac{k}{p}}
       \sqrt{k\log(eM/k)}
       +C\frac{k}{\sqrt p}\\
  &\le C\left[
        k\sqrt{\frac{\log p}{p}}+\frac{k}{\sqrt p}
       \right].
\end{align*}

The cone is full dimensional. Indeed, $u_\varnothing=u_\star$, and for $D=\{j\}$ the vector $u_{\{j\}}$ is a nonzero normalization of
\[
  u_\star-\frac{1}{\|b_j\|_2^2}b_j.
\]
Thus every $b_j$ belongs to the linear span of $u_\star$ and $u_{\{j\}}$. The tight-frame identity implies that the $b_j$ span $\R^p$. The cone is pointed because
\[
  \inf_{|D|\le k}\langle u_D,u_\star\rangle
  \ge1/\sqrt2>0.
\]
It is polyhedral because it has finitely many generators. \Cref{lem:cone-norm} therefore realizes it as the descent cone of a polyhedral norm.

\subsection{Heavy-tailed design}

Let $J$ be uniform on $[M]$, let $\sigma$ be a Rademacher sign independent of $(J,L)$, and set
\[
  \rho=L/\sqrt{\E L^2},
  \qquad
  Z_{\HT}=\sigma\rho b_J.
\]
Then
\[
  \E Z_{\HT}=0,
  \qquad
  \E Z_{\HT}Z_{\HT}^\top
  =\E\rho^2\frac1M\sum_j b_j b_j^\top=I_p.
\]
Since $\rho\ge(\E L^2)^{-1/2}$ and the slab property in \eqref{eq:frame-properties} holds,
\[
  \inf_{v\in\Sph^{p-1}}
  \Pp\left(
    |\langle Z_{\HT},v\rangle|
    \ge a_0/\sqrt{\E L^2}
  \right)
  \ge b_0.
\]
For the Gaussian law, $\langle G,v\rangle\sim\Normal(0,1)$ for every unit $v$.  Thus the explicit common choice
\[
  \alpha_0=\min\left\{\frac{a_0}{\sqrt{\E L^2}},\frac12\right\},
  \qquad
  \beta_0=\min\left\{b_0,\Pp(|G_1|\ge\alpha_0)\right\}
\]
gives \eqref{eq:common-sb} for both designs.

The law of $Z_{\HT}$ has finite moments of every order. If $v\ne0$, the tight-frame identity gives
\[
  \frac1M\sum_j\langle b_j,v\rangle^2=\|v\|_2^2>0,
\]
so some row type has nonzero inner product with $v$. Conditioning on that type shows
\[
  \E e^{s|\langle Z_{\HT},v\rangle|}=\infty
  \qquad(s>0).
\]

Given any sample of size $n\le k$, let $D$ be the set of distinct values among $J_1,\ldots,J_n$. Then $u_D\in A_{p,k}$ and
\[
  \langle Z_{\HT,i},u_D\rangle
  =\sigma_i\rho_i\langle b_{J_i},u_D\rangle=0
\]
for every $i$. This proves deterministic failure. The Gaussian claim follows directly from \eqref{eq:gordon-app}.

\subsection{Constant-width Gaussian specialization}

Fix a sufficiently small absolute $a>0$ and take
\[
  k=\left\lfloor a\sqrt{\frac{p}{\log p}}\right\rfloor.
\]
For all sufficiently large $p$, this $k$ is admissible in \Cref{thm:isotropic}, and \eqref{eq:isotropic-width} gives
\[
  w(A_{p,k})
  \le C\left(a+\frac{a}{\sqrt{\log p}}\right)
  \le C_\star
\]
for an absolute $C_\star$.  Let $t=\sqrt{2\log(1/\delta)}$.  If
\[
  n\ge\max\{2,16C_\star^2,32\log(1/\delta)\},
\]
then $\sqrt{n-1}\ge\sqrt{n/2}$, $C_\star\le\sqrt n/4$, and $t\le\sqrt n/4$.  Therefore \eqref{eq:gauss-lower} yields
\[
  \frac1{\sqrt n}\inf_{u\in A_{p,k}}\|X_{\Gauss}u\|_2
  \ge\frac1{\sqrt2}-\frac12=:c_\star>0
\]
with probability at least $1-\delta$.  Hence $\Lam_n(X_{\Gauss};A_{p,k})\ge c_\star^2$.  Enlarging one universal constant converts the displayed condition on $n$ to $n\ge C_G[1+\log(1/\delta)]$.  The heavy-tailed failure holds on every path for $n\le k$, and, after reducing $a$ by a factor of two to absorb the floor, $k\ge c_H\sqrt{p/\log p}$.  This proves \Cref{cor:constant-isotropic}.

\subsection{Radius and affine dimension}

The radius upper bound follows from \eqref{eq:cap-cone-radius}:
\[
  r(A_{p,k})\le C\sqrt{k/p}.
\]
Choose a set $D$ of cardinality $k$. By \eqref{eq:distance-exact},
\[
  \|u_D-u_\star\|_2\ge c\sqrt{k/p},
\]
so every enclosing ball has radius at least half this distance. Therefore
\[
  r(A_{p,k})^2\asymp k/p.
\]
The cone $C_{p,k}$ is full dimensional.  Its intersection with the sphere therefore contains a nonempty open spherical patch.  If this patch lay in a proper affine hyperplane $\{x:\langle a,x\rangle=b\}$, then the affine function $x\mapsto\langle a,x\rangle-b$ would vanish on an open subset of the sphere.  Differentiating along all tangent directions on that subset forces $a$ to be parallel to every point in the patch, which is impossible unless $a=0$ and then $b=0$.  Hence no proper affine hyperplane contains the patch, and $q(A_{p,k})=p$.
\end{proof}

\section{Smooth-density robustness}\label{app:smooth}

\begin{proof}
The convolution in \eqref{eq:smooth} is centered and isotropic.  Write $c_\tau=\sqrt{1+\tau^2}$ and let $\phi_\tau$ be the density of $\tau H$.  The smoothed row has density
\[
  f_\tau(x)=c_\tau^p\E\phi_\tau(c_\tau x-Z_{\HT}),
\]
which is strictly positive for every $x\in\R^p$.  Every derivative of a Gaussian density is a polynomial times that density and is bounded on $\R^p$.  Dominated differentiation under the expectation therefore gives $f_\tau\in C^\infty(\R^p)$.  Polynomial moments remain finite.

We also verify the stronger directional heavy-tail claim.  Fix $v\ne0$.  Tightness of the frame provides a $j$ with $c_j=\langle b_j,v\rangle\ne0$.  Conditional on $J=j$, and on the independent event $|\tau\langle H,v\rangle|\le1$, which has positive probability,
\[
  |\sigma\rho c_j+\tau\langle H,v\rangle|
  \ge |c_j|\rho-1.
\]
The event is independent of $\rho$.  Hence, after the harmless factor $1/\sqrt{1+\tau^2}$, every positive exponential moment of $|\langle Z_{\HT,\tau},v\rangle|$ is infinite.

Let $\alpha_0,\beta_0$ be global small-ball constants for $Z_{\HT}$. Choose $\tau_0>0$ such that
\[
  \Pp(|\tau H_1|\le\alpha_0/2)\ge1/2
  \qquad(0<\tau\le\tau_0).
\]
For a unit $v$, the heavy and Gaussian components are independent. On the event
\[
  |\langle Z_{\HT},v\rangle|\ge\alpha_0,
  \qquad
  |\tau\langle H,v\rangle|\le\alpha_0/2,
\]
one has
\[
  |\langle Z_{\HT,\tau},v\rangle|
  \ge\frac{\alpha_0}{2\sqrt{1+\tau_0^2}}.
\]
Thus one may take $\bar\alpha=\alpha_0/(2\sqrt{1+\tau_0^2})$ and $\bar\beta=\beta_0/2$.

For each fixed sample size $n\le k$, let $D$ be its latent row-type set. The direction $u_D$ depends only on the $J_i$ and is independent of the Gaussian perturbations. Since the heavy component is annihilated,
\[
  \Lam_n(X_{\HT,\tau};A_{p,k})
  \le\frac1n\sum_{i=1}^n\langle Z_{\HT,\tau,i},u_D\rangle^2
  =\frac{\tau^2}{1+\tau^2}\frac1n
    \sum_{i=1}^n\langle H_i,u_D\rangle^2.
\]
Conditionally on $(J_1,\ldots,J_n)$, the final sum has the $\chi_n^2$ law.  Thus $\chi_n^2\le2n$ implies $\Lam_n\le2\tau^2$.
For completeness, Markov's inequality and
$\E e^{\chi_n^2/4}=2^{n/2}$ give
\[
  \Pp(\chi_n^2>2n)
  \le e^{-n/2}2^{n/2}
  =\exp\left[-\frac{1-\log2}{2}n\right].
\]
This proves \eqref{eq:smooth-upper} with $c_s=(1-\log2)/2$.
\end{proof}

\section{The dimension--radius bound}\label{app:radius}

\begin{proof}
Choose a countable dense subset $A_0\subseteq A$.  For every matrix $X$ and every $h\in\R^p$, continuity gives
\[
  \inf_{u\in A_0}\|Xu\|_2=\inf_{u\in A}\|Xu\|_2,
  \qquad
  \sup_{u\in A_0}\langle h,u\rangle
  =\sup_{u\in A}\langle h,u\rangle.
\]
Moreover, $\SB_\alpha(P,A_0)\ge \SB_\alpha(P,A)\ge\beta$.  We may therefore apply the countable, measurable form of the small-ball lemma to $A_0$; the displayed identities transfer its conclusion back to $A$.

Because $A$ is bounded and the ambient space is finite dimensional, the continuous coercive function $a\mapsto\sup_{u\in A}\|u-a\|_2$ attains its minimum (passing to the closure of $A$ does not change the supremum).  Orthogonally projecting any minimizer onto $\aff(A)$ cannot increase its distance to any $u\in A$.  We may therefore choose a minimum enclosing center $a\in\aff(A)$.  Let
\[
  L_A=\spanop(A-a),
  \qquad \dim L_A=q(A).
\]
For
\[
  H_n=\frac1{\sqrt n}\sum_{i=1}^n\eps_i Z_i,
\]
the Rademacher signs give $\E H_n=0$.  For each realization,
$\sup_{u\in A}\langle H_n,u\rangle=\langle H_n,a\rangle+\sup_{u\in A}\langle H_n,u-a\rangle$, and $\E\langle H_n,a\rangle=0$.  Hence
\begin{align*}
  W_n(A;P)
  &=\E\sup_{u\in A}\langle H_n,u\rangle
    =\E\sup_{u\in A}\langle H_n,u-a\rangle\\
  &\le r(A)\E\|P_{L_A} H_n\|_2\\
  &\le r(A)\sqrt{\E\|P_{L_A} H_n\|_2^2}.
\end{align*}
Isotropy gives $\E H_n H_n^\top=I_p$, and therefore
\[
  \E\|P_{L_A} H_n\|_2^2
  =\tr(P_{L_A})=q(A).
\]
Thus
\[
  W_n(A;P)\le r(A)\sqrt{q(A)}.
\]
Apply \Cref{lem:small-ball-standard} with $\xi=\alpha/2$. Since $\SB_\alpha(P,A)\ge\beta$, this yields \eqref{eq:radius-bound}.

To derive \eqref{eq:radius-sample}, choose $t=\sqrt{2\log(1/\delta)}$ and require separately
\[
  2r(A)\sqrt{q(A)}\le\frac{\alpha\beta}{8}\sqrt n,
  \qquad
  \frac{\alpha t}{2}\le\frac{\alpha\beta}{8}\sqrt n.
\]
Under the resulting sample-size condition, the right side of \eqref{eq:radius-bound} is at least $\alpha\beta\sqrt n/4$, so $\Lam_n(X;A)\ge\alpha^2\beta^2/16$.

Finally, Gaussian width is always at most the dimension--radius scale. Indeed,
\[
  w(A)
  =\E\sup_{u\in A}\langle g,u-a\rangle
  \le r(A)\E\|P_{L_A}g\|_2
  \le r(A)\sqrt{q(A)}.
\]
The matching identities for $A_{p,k}$ were proved in the preceding appendix.
\end{proof}

\section{Convex-recovery consequences}\label{app:recovery}

\begin{proof}
Suppose $A=\cD(\cR,\theta^\star)\cap\Sph^{p-1}$ and $Xu=0$ for some $u\in A$. By the definition of the descent cone, there is $t>0$ such that
\[
  \cR(\theta^\star+tu)\le\cR(\theta^\star).
\]
At the same time $X(\theta^\star+tu)=X\theta^\star$. Hence $\theta^\star$ cannot be the unique solution of \eqref{eq:convex-program}. This proves the exact-recovery statements in \Cref{sec:consequences}.

For the smoothed construction, \eqref{eq:smooth-upper} supplies a unit descent direction $u$ with $\|Xu\|_2\le\sqrt2\tau\sqrt n$.  In normalized units,
\[
  \kappa_n(X;A)
  :=\frac1{\sqrt n}\inf_{v\in A}\|Xv\|_2
  =\sqrt{\Lam_n(X;A)}
  \le\sqrt2\tau.
\]
Equivalently, the realized-design inverse modulus
\[
  K_n(X;A):=\frac1{\kappa_n(X;A)}
\]
is at least $1/(\sqrt2\tau)$, with the convention $1/0=\infty$, and the quadratic modulus $1/\Lam_n$ is at least $1/(2\tau^2)$.  This conclusion is deterministic after $X$ is realized.  In particular, the witness $u$ and a descending comparison point $\theta_1=\theta^\star+su$ may depend on that realization.  Their normalized observation distance is at most $\sqrt2\tau s$, so the statement controls the conditioning of the realized inverse problem; it does not assert a minimax lower bound for two parameters fixed before sampling.
\end{proof}